\documentclass{article}
\usepackage{amsmath, amsfonts, amsthm}
\usepackage{times}
\usepackage{dsfont}
\usepackage{fullpage}

\newtheorem{theorem}{Theorem}
\newtheorem{lemma}[theorem]{Lemma}
\newtheorem{corollary}[theorem]{Corollary}
\newtheorem{proposition}[theorem]{Proposition}
\newtheorem{observation}[theorem]{Observation}
\newtheorem{definition}{Definition}

\DeclareMathOperator{\VC}{VC}

\begin{document}

\title{On the Complexity of Learning from Label Proportions}

\author{Benjamin Fish\\Microsoft Research\\benjamin.fish@microsoft.com
  \and Lev Reyzin\\University of Illinois at Chicago\\lreyzin@uic.edu}

\date{}


\maketitle

\begin{abstract}
In the problem of learning with label proportions, which we call LLP learning, 
the training data is unlabeled, and only the proportions of examples receiving each label are given.  The goal is to learn a hypothesis that predicts the 
proportions of labels on the distribution underlying the sample.  This model of learning is applicable to a wide variety of settings, including predicting the number of votes for candidates in political elections from polls.  

In this paper, we formally define this class and resolve foundational questions regarding the computational complexity of LLP and characterize its relationship
to PAC learning. 
Among our results, we show, perhaps surprisingly, that for finite VC classes what can be efficiently LLP learned is a strict subset of what can be leaned efficiently in PAC, under standard complexity assumptions.   We also show that there exist classes of functions whose learnability in LLP is independent of ZFC, the standard set theoretic axioms.
This implies that LLP learning cannot be easily characterized (like PAC by VC dimension).


\end{abstract}

\section{Introduction}

In this paper, we investigate the complexity of the learning problem of estimating the proportion of labels for a given set of instances.  For example, this problem appears when predicting the proportion of votes for a given candidate~\cite{kuck2005learning}; correctly predicting how each individual votes is not required, only which candidate will win.  Variants of this problem also appear in many other domains, including in consumer marketing~\cite{chen2006learning}, medicine and other health domains~\cite{hernandez2013learning,wojtusiak2011using}, image processing~\cite{kuck2005learning}, physical processes~\cite{musicant2007supervised}, fraud detection~\cite{rueping2010svm}, manufacturing~\cite{stolpe2011learning}, and voting networks~\cite{FishHR16}.

In classical Probably Approximately Correct (PAC) learning~\cite{valiant84learnable}, we are given labeled data instances from a distribution, and in the idealized case, must find a function that labels all of the data consistent with the observations.  In less constrained settings, the goal is to find a function of low error, or at least of error as low as possible on the data presented to the algorithm.  There is substantial literature on classical PAC learning outside the scope of this work; see e.g.~\cite{shalev2014understanding} for a survey.  Once the classifier is found, it is easy to find the proportion of instances with a given label by invoking the classifier on the instances.  Algorithms for estimating the proportion of labels with labeled data have been introduced before, for example by Iyer et al.~\cite{iyer2014maximum}.

However, getting instances with attached labels, as assumed in classical PAC learning, is often difficult. Sometimes this is due to limits on the measurement process~\cite{hernandez2013learning,kuck2005learning,musicant2007supervised,stolpe2011learning}. At other times, before datasets are released, labels are purposely detached from their instances in order to maintain privacy~\cite{chen2006learning,rueping2010svm,wojtusiak2011using}.  Instead, only the proportion of labels are given for a group of sample instances.  For example, in estimating who will win an election, pre-election polls only release the percentage of people planning to vote for a given candidate.  Quadrianto et al.~\cite{quadrianto2009estimating} give several other examples where the only data available is of this form.

The goal is then to learn a classifier from a hypothesis class that is able to correctly predict the proportions of labels from a hidden distribution using a training set which consists of a set of instances and the proportions of labels of that set of instances.  This is the learning-theoretic problem we formalize and tackle in this paper.
The proportion of labels may be inferred by first finding a classifier that predicts the labels for each instance~\cite{patrini2014almost,quadrianto2009estimating,rueping2010svm,yu2013propto}.  Alternatively, Iyer et al.~\cite{iyer2016privacy} propose inferring
the proportion of labels directly.


Yu et al.~\cite{yu2014learning} introduce a version of a model for learning from label proportions.  In their model, each bag of examples comes with the proportions of each label in that bag, and each bag is drawn i.i.d.~from a distribution over bags.  They give some of the first sample complexity guarantees.  Another approach is where the examples are drawn~i.i.d., but the bags may be an arbitrary partition of the examples, as in~\cite{rueping2010svm,stolpe2011learning}.  Compared to these `bag' models, our model of learning from label proportions corresponds to the `one-bag case' with binary labels, where each example is drawn i.i.d.~from an arbitrary distribution.  However, as we demonstrate, this model is already interesting to study.  We formalize this as a PAC-like learning model, which allows us to compare the difficult of learning a hypothesis class in classical PAC learning to learning a hypothesis class in this model, which we call Learnable from Label Proportions (LLP).

In particular, we give the following results, including the first computational hardness results for learning label proportions.
After formally defining the model in Section~\ref{sec:model}, we start in Section~\ref{sec:ind} by pointing out that LLP satisfies a natural uniform convergence property, meaning that any class with finite VC dimension is learnable in this model.  However, LLP is not characterized by finite VC dimension:  We give a simple class of functions with infinite VC dimension that is learnable in this model.  In addition, we give a class of functions whose learnability under label proportions is independent
of ZFC, implying that learnability in this model does not admit any simple characterization like VC dimension.

In Section~\ref{sec:subset}, we compare efficient PAC to efficient LLP, where the learning algorithm must take only polynomial time in the size of its input.  We show that for classes with finite VC dimension, if it is efficiently learnable from label proportions, it is also efficiently properly PAC learnable.  We then go on to show that learning from label proportions can be harder than PAC learning in Section~\ref{sec:hardness}:  classes that are efficiently PAC learnable like parities and monotone disjunctions are hard to learn in this setting.  Finally, in Section~\ref{sec:learnable} we give some positive results indicating cases where it is possible to PAC learn from label proportions.  We also show that $n$-dimensional half-spaces over the boolean cube are learnable from label proportions under the uniform distribution.

\section{Model and sample complexity}\label{sec:model}

For $c$ a function $\{0,1\}^n\rightarrow\{0,1\}$ and $D$ a distribution over the domain of $c$, we will call $p_c$ the percentage of positive labels in this distribution, i.e.~$p_c = \mathbb{P}_{x\sim D}[c(x) = 1]$.
For a given sample $S$, we call the percentage labeled positively as $\hat{p}_c = \frac{1}{|S|}\sum_{s\in S}c(x)$.  Where clear, we will abbreviate these as $p$ and $\hat{p}$ respectively.

In this setting, each example $x$ drawn from $D$ has a hidden label $c(x)$, but the learning algorithm does not get to see examples with labels.  Instead, the algorithm only gets to see the set of unlabeled examples $S$ and $\hat{p}_c$, the percentage of $S$ labeled positively by $c$.  The goal is to find a function $h$ in a hypothesis class $H$ such that $p_c$ should be close to $p_h$ with high probability.

\begin{definition}[Learnable from Label Proportions (LLP Learnable)]\label{defn:pac}
A class of functions $H$ is {\bf learnable from label proportions (LLP learnable)} if there is an algorithm $A$ such that for every target function $c$ in $H$, any distribution $D$ over $\{0,1\}^n$, and for any $\epsilon,\delta >0$, given $m\ge poly(1/\epsilon,1/\delta, n, size(c))$ examples drawn i.i.d.~from $D$ and $\hat{p}_c$, returns a hypothesis $h$ in $H$ such that 
\[\mathbb{P}[|p_c-p_h| \le\epsilon] \ge 1-\delta.\]
\end{definition}

%

If in addition, $A$ is an efficient algorithm (i.e.\ running in time polynomial in the size of its input), then we call such a class \textbf{efficiently learnable
from label proportions (efficiently LLP learnable)}.

We refer to the classes of LLP learnable and efficiently LLP learnable functions as \textbf{LLP} and \textbf{efficient LLP}, respectively.

In general, we may consider agnostic or improper versions of this PAC model.  However, improper learning from the class of all functions here is very easy:  We can efficiently learn with a sample complexity that only depends on $\epsilon$ and $\delta$:

\begin{observation}
The sample complexity for improper PAC learning from label proportions is $O\left(\frac{\ln(1/\delta)}{\epsilon^2}\right)$.
\end{observation}
\begin{proof}[Proof outline]
In improper learning, it is easy to find a function $h^*$ so that not only does $\hat{p}_{h^*} = \hat{p}$, but also $p_{h^*} = \hat{p}$:  e.g.~$h^*$ may be a randomized function that on any input returns $1$ with probability $\hat{p}$ and $0$ otherwise.  Then $p_{h^*} = \hat{p}$ and a Hoeffding bound implies that $\hat{p}$ is close to $p$.
\end{proof}

For example, if the task is to predict the proportion of votes for a given candidate using only a single poll, \emph{improper} learning in this model is easy simply by virtue of the fact that $\hat{p}$ is an unbiased estimator for $p$.  However, the hypothesis $h^*$ described above will not be a realistic model of voting.  So proper learning corresponds to finding a realistic model of voting, one which describes a relationship between examples and labels, that also predicts the proportion of votes correctly.  For this reason, for the remainder of this paper, we will only consider \emph{proper} PAC learning from label proportions.  

Definition~\ref{defn:pac} is a distribution-free setting, but when the distribution is known, sample complexity also may be independent of the VC-dimension.

\begin{observation}\label{obs:gap_sample_complexity}
Let $D$ be a known distribution.  Let \[\beta = \inf_{\substack{h,h'\in H:\\ h\neq h'}} |p_h-p_{h'}|.\]  Then the sample complexity for PAC learning from label proportions the hypothesis class $H$ is $O\left(\frac{\ln(1/\delta)}{\beta^2}\right)$.
\end{observation}
\begin{proof}[Proof outline]
Here, we can use another Hoeffding bound to get that with high probability, $\hat{p}$ is within $\beta/2$ of $p_c$, for $c$ the target hypothesis.  But the definition of $\beta$ implies that there is exactly one value $p_{c^*}$ in $\{p_c: c\in H\}$ such that $\hat{p}$ is closer to $p_{c^*}$ than any other value in $\{p_c: c\in H\}$.  Then with high probability $p_c=p_{c^*}$.  Thus an algorithm may output any $h$ such that $p_{h}=p_{c^*}$.
\end{proof}

In Section~\ref{sec:learnable}, we return to distribution-specific learning.


\section{Comparing PAC to LLP}\label{sec:ind}

We start by considering sample complexity in the distribution-free setting.  Is learning from label proportions harder or easier than PAC learning?  We first give a uniform convergence result that implies that if a class has finite VC dimension, than it is learnable from label proportions.  Despite this result, unlike the PAC setting, VC dimension does \emph{not} characterize learning from label proportions.  Perhaps surprisingly, we show that there are classes with infinite VC dimension that are learnable from label proportions and those whose learnability from label proportions are independent from ZFC (for more, see below).

First, we prove a uniform convergence bound.  Following the proof of the equivalent bounds in PAC learning under an arbitrary loss function (see Shalev-Shwartz and Ben-David~\cite{shalev2014understanding}), we can show the same bounds also hold here.  Namely, we can use the VC dimension of a hypothesis class $H$ to bound generalization error.  We denote this quantity by $\VC(H)$.  In particular, we have:
\begin{theorem}[Uniform convergence]\label{thm:vc_gen_bound}
For target function $c\in H$, with probability at least $1-\delta$, for all $h\in H$,
\[ |p_h-p_c| \le |\hat{p}_h-\hat{p}_c| + \sqrt{\frac{8\VC(H)\log(em/\VC(H))}{m}} + \sqrt{\frac{2\log(4/\delta)}{m}}. \]
\end{theorem}
\begin{proof}
Consider the empirical loss $|\hat{p}_h-\hat{p}_c| = \left|\frac{1}{m}\sum_{i=1}^m h(x_i)-c(x_i) \right|$ for sample $S=\{x_1,\ldots,x_m\}$.  Unfortunately its expectation $\mathbb{E}_{S\sim D^m}[|\hat{p}_h -\hat{p}_c|]$ \emph{does not} equal the distributional loss $|p_h-p_c| = |\mathbb{E}_{x\sim D}[h(x)] - \mathbb{E}_{x\sim D}[c(x)]|$, so instead, we consider convergence of the \emph{signed} loss instead:  $\ell(h,x) := h(x)-c(x)$, $\mathcal{L}_D(h) := \mathbb{E}_{x\sim D}[\ell(h,x)]$, and $\mathcal{L}_S(h) := \frac{1}{m}\sum_{i=1}^m \ell(h,x_i)$.  This is a loss function with $|\ell(h,x)| \le 1$, so we can bound generalization error by the empirical Rademacher complexity (see e.g.\ Shalev-Shwartz and Ben-David \cite{shalev2014understanding}):  For $R(\cdot)$ the Rademacher complexity, with probability of at least $1-\delta$, for all $h\in H$,
\[\mathcal{L}_D(h) - \mathcal{L}_S(h) \le 2\mathbb{E}_{z_i\sim D}[R(\{(\ell(h,z_1),\ldots,\ell(h,z_m))\mid h\in H\})] + \sqrt{\frac{2\ln(2/\delta)}{m}}.\]
For convenience, let $d=\VC(H)$.
The Sauer-Shelah lemma implies that $|\{(h(z_1),\ldots,h(z_m)) | h\in H\}| \le \left(\frac{e m}{d}\right)^{d}$,
we also have $|\{(\ell(h,z_1),\ldots,\ell(h,z_m))\mid h\in H\}| \le \left(\frac{em}{d}\right)^{d}$.  Again since $|\ell(h,x)| \le 1$, Massart's lemma then implies that $R(\{(\ell(h,z_1),\ldots,\ell(h,z_m))\mid h\in H\}) \le \sqrt{\frac{2d\log(em/d)}{m}}$.  Thus with probability of at least $1-\delta/2$, for all $h\in H$,
\[\mathcal{L}_D(h) - \mathcal{L}_S(h) \le 2\sqrt{\frac{2d\log(em/d)}{m}} + \sqrt{\frac{2\ln(4/\delta)}{m}}.\]
Repeating this argument for $-\ell(h,z)$, and then applying the union bound, we conclude
\[|\mathcal{L}_D(h) - \mathcal{L}_S(h)| \le 2\sqrt{\frac{2d\log(em/d)}{m}} + \sqrt{\frac{2\ln(4/\delta)}{m}}.\]
But if the signed losses are bounded to each other, than so must the unsigned losses be bounded to each other:
\[||p_h-p_c| - |\hat{p}_h-\hat{p}_c|| \le 2\sqrt{\frac{2d\log(em/d)}{m}} + \sqrt{\frac{2\ln(4/\delta)}{m}}.\]
\end{proof}

This immediately implies that PAC learnable classes form a subset of classes learnable from label proportions.  In the remainder of this section, we show first that it is a strict subset.  We then show that we can find a hypothesis class whose learnability is independent of the Zermelo-Fraenkel set theory (ZFC), the common
axiomatic system of mathematics which, following the work of Ben-David et al.~\cite{BenDavidHMSY19} was suggested by Reyzin~\cite{Reyzin19}
would apply to other models of machine learning, such as this one.  


To demonstrate this, we make use of
a recent result by Ben-David~et~al.~\cite{BenDavidHMSY19}, who defined EMX-learnability and gave a class
of functions whose learnability is independent of ZFC.
First, we give the definition of EMX learnability.

\begin{definition}[Ben-David~et~al.~\cite{BenDavidHMSY19}]
An $(\epsilon,\delta)$-EMX learner for a class of functions $C$ is an algorithm that, for some integer $m = m(\epsilon,\delta)$,
produces a hypothesis $c_S \in C$ for which
$$
\mathbb{P}_{S \sim D^m}[\mathbb{E}_D[c_S] \le \sup_{c \in C}\mathbb{E}_D[c] - \epsilon] \le \delta
$$
for every distribution that is finitely-supported over the $\sigma$-algebra over all subsets of the input space.
\end{definition}

This definition asks the learner to produce a (proper) function having as high distributional weight as possible.  
Ben-David~et~al.~\cite{BenDavidHMSY19} give the following suprising result.

\begin{theorem}[Ben-David~et~al.~\cite{BenDavidHMSY19}]\label{thm:emx_independence}
EMX learnability of finite subsets of $[0,1]$ over finitely supported distributions (over the unit interval) is independent of ZFC.
\end{theorem}

Now, we demonstrate that in some cases, EMX learning suffices to be able to learn from label proportions:

\begin{theorem}\label{thm:interval_reduction}
Consider a space $Z$ equipped with a total order.  For any interval of $Z$ (induced by the total order), suppose finite subsets of that interval are EMX learnable over finitely supported distributions.  Then finite subsets of $Z$ are learnable from label proportions over finitely supported distributions.
\end{theorem}
\begin{proof}
To learn finite subsets of $Z$ from label proportions, we first show that even though the finite subset may be arbitrarily large, it is contained in a small number of intervals of $Z$.  We then show how an EMX learner can achieve any label proportion
on $Z$ (within a factor of $\epsilon$) using these intervals.

First, define \emph{heavy} points as those whose distributional mass at least $\epsilon/4$ -- note there can be no more than $4/\epsilon$ such points.
Let the remaining points be called \emph{light}.  Let $W_L$ be the distributional weight on the light
points.  There must be a subset of the heavy points whose weight adds up to within $W_L$ of the target proportion
$p$.  Between these heavy points there are $4/\epsilon + 1$ intervals of $Z$ containing light points, so there are also at most $4/\epsilon + 1$ sub-intervals that combined with the subset of the heavy points reach within $\epsilon/4$ of the target proportion:  each additional light point included in a sub-interval can add only at most $\epsilon/4$ to the total mass included.

Hence, there exists a union of at most $4/\epsilon + 1$ intervals containing light points, that 
together with
a subset of at most $4/\epsilon$ heavy points (around which we can also add intervals around the single points), whose weight adds up to the right proportion (within $\epsilon/4$).
Since the VC dimension $d$ of a union of $k= 8/\epsilon + 1$ intervals scales as $2k+1$, 
Theorem \ref{thm:vc_gen_bound} tells us that $\tilde{O}({d/\epsilon^2}) = \tilde{O}({1/\epsilon^3})$ samples suffice to find such intervals via empirical risk minimization on the sample to within $\epsilon/4$, meaning we can find $k$ such intervals with proportion within $\epsilon/2$ of the true proportion.

Finally, for each interval, we use an EMX learner to find a finite subset within that interval with mass approximating the total mass of the interval, within sufficiently small error; e.g.\ within \ ${\epsilon^2}/{18}$ suffices.  Then the union of these finite subsets will be the finite subset of $Z$ with mass within $\epsilon$ of the true proportion (except with probability $\delta$).  Since there are at most 
$8/\epsilon + 1$ such intervals to approximate, the union bound implies that the total error contributed in this part is also bounded by
$\epsilon/2$.  This is sufficient to achieve LLP learning.
\end{proof}

When $Z$ is the unit interval over the reals, this means that EMX learning finite subsets is equivalent to LLP learning finite subsets, and because EMX learning finite subsets is independent of ZFC, so is LLP learning:
\begin{corollary}\label{thm:llpind}
LLP learnability of finite subsets of $[0,1]$ over finitely supported distributions is independent of ZFC.
\end{corollary}
\begin{proof}
Using Theorem \ref{thm:emx_independence}, it suffices to show that finite subsets of $[0,1]$ over finitely supported distributions over the unit interval are PAC learnable from label proportions if any only if they are EMX learnable.  The ``if'' direction follows from Theorem~\ref{thm:interval_reduction}.

Now for the other direction, since we know there exists a finite subset of the unit interval that contains the entire distributional mass, i.e.\ has true proportion $1$, we can feed a learner from label proportions $\hat{p} = 1$, along with the sample.  If this learner can successfully LLP learn finite subsets of $[0,1]$, this learns a finite subset that maximizes the proportion, as desired.

Finally, we note that while the polynomial sample complexity requirement of LLP is not present in EMX, this doesn't cause a problem.  This is
because we can always define a representation for which the finite sample complexity bound is polynomial.  While this is not normally sensible, it is sufficient to establish
the existence of a class for which the result holds.
\end{proof}

And when $Z=\mathbb{N}$, Theorem \ref{thm:interval_reduction} implies that finite subsets of $\mathbb{N}$ are learnable from label proportions.  This is because intervals of $\mathbb{N}$ are either isomorphic to itself, or is a finite number of points.  Finite subsets of either of these are EMX learnable~\cite{BenDavidHMSY19}.  But actually, it's worth noting something stronger, that finite subsets of $\mathbb{N}$ are \emph{efficiently} learnable.  This shows that there are classes that are efficiently LLP learnable that are not PAC learnable.

\begin{corollary}\label{cor:efficient_llp_example}
Finite subsets of $\mathbb{N}$ are efficiently learnable from label proportions over finitely supported distributions.
\end{corollary}
Examining the proof of Theorem \ref{thm:interval_reduction}, it suffices to show that LLP learning unions of intervals and then EMX learning each of those interval can be done efficiently.  But note that in the case of $\mathbb{N}$, first finding a set of intervals amongst a size $m$ set of sample points using Theorem \ref{thm:vc_gen_bound} is equivalent to choosing a subset of those points whose mass is within $\epsilon/4$ of the empirical proportion by empirical risk minimization.  Given that subset of points, EMX learning for each interval is just selecting all of the subset, so thus it suffices to perform ERM on the original set of $m$ sample points.  This can be done efficiently in time polynomial in $1/\epsilon$ using a dynamic program for subset sum.

\section{Comparing efficient PAC to efficient LLP}\label{sec:subset}
In this section, we consider how requiring efficiency of the learning algorithm changes the relationship between LLP and PAC.  In the previous section we showed that it is easier to LLP learn than it is to PAC learn, in the sense that PAC is a strict subset of LLP.  In this section, we show that this relationship does not hold between efficient LLP and efficient PAC.  From Corollary~\ref{cor:efficient_llp_example}, we already know that there are hypothesis classes with infinite VC dimension that are efficiently LLP learnable but not efficiently PAC learnable, but here we show that there are classes that are hard to LLP learn even though they are efficiently PAC learnable.  Moreover, the only classes that are efficiently learnable from label proportions that are not efficiently PAC learnable have infinite VC dimension:



\begin{theorem}
Suppose $\textup{NP} \neq \textup{RP}$.  Then if a hypothesis class $H$ with finite VC dimension is efficiently learnable from label proportions, it is also efficiently (properly) PAC learnable.
\end{theorem}
\begin{proof}
Let $H$ be learnable from label proportions by some efficient oracle $A$, and $f$ the polynomial sample size required by this oracle.  We now give an efficient algorithm for PAC learning $H$.  Given $\epsilon, \delta>0$, draw $m$ samples from the unknown distribution $D$, with $m$ to be determined later.  Call the set $S$ of unique inputs $x_1,\ldots,x_m$ and their labels $c(x_1),\ldots,c(x_m)$ for hidden target function $c$.  Let $k$ be the number of positive labels $\sum_j c(x'_j)$.
Define a new distribution $D'$ as the following:
\[D'(x) =
\left\{\arraycolsep=2.5pt\def\arraystretch{2.0}
\begin{array}{cc}
\frac{m}{km+m-k} & \text{if }x\in S \text{ and } c(x)=1 \\
\frac{1}{km+m-k} & \text{if }x\in S \text{ and } c(x)=0 \\
0 & \text{otherwise}
\end{array}\right\}.\]

Let $\epsilon' = 1/(2m^2)$ and $\delta'=\delta$.
Draw $m'=f(1/\epsilon', 1/\delta')$ samples $x'_j$ from $D'$ and label each as $c(x'_j)$.  We give to the oracle as input $\epsilon'$, $\delta'$, and the examples $x'_j$, along with the proportion of positive labels $\hat{p}=\frac{k}{m'}$.  Then with probability at least $1-\delta$ the oracle returns a hypothesis $c^*$  such that 
\[|p_{c^*}-p_c| < \frac{1}{2m^2}.\]  The smallest non-zero probability mass in $D'$, however, is 
\[\frac{1}{km+m-k} \ge \frac{1}{m^2},\] minimized when $k=m$.  Thus $p_{c^*}=p_c$.

We now show that $c^*=c$ when restricted to the points $x_1,\ldots,x_m$.  Suppose there is a point $x_i$ such that $c^*(x_i)\neq c(x_i)$ where $c(x_i)=1$.  Then in order to have $p_{c^*}=p_c$ while $c^*(x_i)=0$, at least $m$ points labeled $0$ by $c$ must be labeled positively by $c^*$, since $D'$ places (proportional to) $m$ weight on positively labeled points and only unit weight on negative points.  This is a contradiction, as there are only $m$ total points.  Similarly, if $c(x_i)=0$ and $c^*(x_i)=1$, there must be $m$ points labeled $0$ by $c^*$ that are labeled $1$ by $c$, but again there are only $m$ distinct points.  Thus $c$ and $c^*$ must agree on all $m$ points, i.e.~$c^*$ has zero empirical error.

All that remains is to check that we need only a polynomial sample size to use uniform convergence (Theorem~\ref{thm:vc_gen_bound}).  This only requires $\VC(H) = \text{poly}(1/\epsilon,1/\delta, n, \text{size}(c))$.  If $H$ is finite, recall that $\VC(H) \le \log|H|$.  But since $\text{size}(c) \ge \log |H|$, we certainly have $\VC(H) = \text{poly}(\text{size}(c))$.  If $H$ is infinite, then the size of the input and classifiers are unbounded.  Then the assumption that $\VC{H} <\infty$ implies that it is also a constant, and therefore the bound given by Theorem~\ref{thm:vc_gen_bound} is indeed a polynomial.
\end{proof}

\subsection{Hardness of learning from label proportions}\label{sec:hardness}

We now give three examples of hypothesis classes that are hard to LLP learn:  parities, monotone disjunctions, and monotone conjunctions.  Parities are functions $h_a:\{0,1\}^n\rightarrow\{0,1\}$ of the form $h_a(x) = a\cdot x \text{ mod }2$, i.e.\ the parity of the bits of a subset of $x$ determined by the bit mask $a$.  Monotone disjuctions are disjunctions on $n$ variables without negations of the literals and likewise monotone conjunctions are conjunctions without negations.  Each of these has linear VC dimension and is efficiently PAC learnable~\cite{shalev2014understanding}.
  

We start by showing that parities are hard to learn from label proportions, even when the subset of the bits is always amongst the first $k$ bits of the input for $k$ poly-logarithmic in $n$.  This class has VC dimension linear in $k$.
To do this, recall in (white-label) noisy PAC learning, each label in the training data is flipped with unknown rate $\eta$.  We assume the algorithm is given as input some $\eta'$, where $\eta \le \eta' < 1/2$ and must only take time polynomial in $\frac{1}{1-2\eta'}$.  Noisy PAC learning parity functions under the uniform distribution is presumed to be hard.  Blum et al.~\cite{BlumKW03} give an $2^{O(n/{\log n)}}$ algorithm, which is the best-current bound.

We now find a specific distribution where PAC learning from label proportions is hard in this sense for parities:
\begin{theorem}\label{thm:parities}
For a hypothesis $c$, Let $D_c$ be the the distribution over $\{0,1\}^n$ that places $\frac{\eta}{2^{n-1}}$ weight on the examples labeled $0$ and $\frac{1-\eta}{2^{n-1}}$  weight on examples labeled $1$.

PAC learning parities from label proportions under $D_c$ is as least as hard as PAC learning unknown parity $c$ with $\eta$ white-label noise under the uniform distribution.
\end{theorem}


\begin{proof}

We use an oracle for PAC learning parities from label proportions under $D_c$ to noisy-PAC learn parities.  We get as input $\eta'$, parameters $\epsilon$ and $\delta$, and some $m$ examples $x_i$, with $m$ to be determined later, with noisy labels $\tilde{\ell}_i$.  When $\tilde{\ell}_i = 1$, with probability $\eta$, the true label $\ell_i=0$ and otherwise $\ell_i=1$.  We may assume that the unknown parity $c$ is non-trivial.  Then under the uniform distribution over $\{0,1\}^n$, for any such parity function, there are $2^{n-1}$ points labeled $1$ and $2^{n-1}$ points labeled $0$.  For any point labeled $0$, the probability that it was drawn from the uniform distribution is $\frac{1}{2^{n-1}}$ and the probability that its label was flipped to $1$ was $\eta$.  Then the probability that an example had $\tilde{\ell}_i=1$ but $\ell_i=0$ is $\frac{\eta}{2^{n-1}}$ and similarly if $\ell_i=1$ the probability is $\frac{1-\eta}{2^{n-1}}$.  Note that this is exactly the distribution $D_c$.  
So if the oracle for PAC learning parities from label proportions is given just the examples where $\tilde{\ell}_i = 1$, the oracle will receive i.i.d.~samples from $D_c$.  We will also give to the oracle $\epsilon'=\frac{1/2-\eta'}{2}$ and $\delta'=\delta/3$.  The expected proportion of these examples given to the oracle is $1-\eta$, but we do not know the true labels nor do we know $\eta$.  So instead, we will invoke this oracle $M+1$ times, with the proportion given to the oracle as each of $0,1/M,\ldots,1$, where $M=\sum_i \tilde{\ell}_i$, i.e.~the number of training examples with noisy label $\tilde{\ell}_i=1$~\footnote{The oracle is undefined when the proportion of positive labels is not the true value $\hat{p}$.  We may assume that the oracle returns an arbitrary hypothesis in this case.}.

If the oracle returns the correct parity $c$, then it should agree in expectation with the noisy labels $\tilde{\ell}_i$ on all but $\eta$ of the examples.  For an incorrect parity $c'$, by the orthonormality of the parity functions, the expected disagreement is $1/2$.  For $h$ the output of the oracle, if smaller than an $\frac{\eta'+1/2}{2}$ fraction of the noisy labels $\tilde{\ell}_i$ disagree with the corresponding label $h(x_i)$, then we return the hypothesis.  Otherwise, we repeat with the next invocation of the oracle.

Let $f$ be the polynomial sample bound for the oracle for PAC learning from label proportions .  First, we need to make sure that the oracle receives at least $f(1/\epsilon',1/\delta')$ examples except with probability at most $\delta/3$.  In expectation, $m/2$ of the examples $x_i$ will have $\tilde{\ell}_i = 1$.  Using a Hoeffding bound, 
\[\mathbb{P}\left[\left|\sum_i \tilde{\ell}_i - m/2\right| > m/4\right] \le 2e^{-m/8}.\]  So the oracle will receive at least $\frac{1}{4}m$ examples (and no more than $\frac{3}{4}m$ examples) except with probability no more than $\delta/3$ so long as $m>8\log(6/\delta)$.  This then means that we require $m > 4\cdot f(1/\epsilon',1/\delta')$ so that $M\ge f(1/\epsilon',1/\delta')$.

Now we need to verify that when the proportion given to the oracle is the correct proportion $\hat{p}_c$, the oracle will return $c$ except with probability at most $\delta/3$.  The oracle is guaranteed to return a parity $h$ such that except with probability $\delta' = \delta/3$, $$|p_h-p_c| \le \epsilon'=\frac{1/2-\eta'}{2}.$$  Using the definition of $D_c$, $p_c=1-\eta$.  If $h\neq c$, then $p_h=1/2$ again by orthonormality.  But then 
\[|p_h-p_c| = |1/2-\eta| > \frac{1/2-\eta'}{2},\] so it must be the case that $h=c$.  Thus at least one of the invocations of the oracle will return the correct parity.

So it remains to show that we will succeed at returning this parity.  If the oracle returns an incorrect parity $h$, again using a Hoeffding bound,
\begin{align*} 
  \mathbb{P}\left[\left|\frac{\sum_{i} \mathds{1}_{h(x_i)\neq\tilde{\ell}_i}}{m} - 1/2\right| \ge \frac{1/2-\eta'}{2}\right] 
&\le  2e^{-\frac{m(1/2-\eta')^2}{2}} \\
&< \frac{1}{M+1}\cdot\frac{\delta}{3}
\end{align*}
when
\[m = \Omega\left(\frac{\log(\frac{M}{\delta})}{(1/2-\eta')^2}\right) = \Omega\left(\frac{\log\left(\frac{1}{(1/2-\eta')\delta}\right)}{(1/2-\eta')^2}\right)\] because $M \le \frac{3}{4}m$, where $\mathds{1}_A$ is the indicator function that is $1$ if $A$ is true and $0$ otherwise.
This implies that for an incorrect hypothesis, whose expected fraction of disagreements with the noisy labels is $1/2$, the empirical fraction is at least $\frac{\eta'+1/2}{2}$, the threshold we had set.  Similarly, for the correct hypothesis, where the expected fraction of disagreements is $\eta<\eta'$, the empirical fraction of disagreements is no more than $\frac{\eta'+1/2}{2}$ except with probability at most $\frac{1}{M+1}\cdot\frac{\delta}{3}$.  This means that all of the tests of the hypothesis succeeds except with probability at most $\delta/3$.  Then setting 
\[m = \Omega\left(\max{\left(\left(\frac{\log\left(\frac{1}{(1/2-\eta')\delta}\right)}{(1/2-\eta')^2}\right),4\cdot f(1/\epsilon',1/\delta')\right)}\right)\] 
suffices so that, with the union bound, the total probability of failure is no more than $\delta$, as required.  
\end{proof}

Consider parity functions on the first $k$ bits, which have VC dimension equal to $k$.  There is no known algorithm for noisy PAC learning parity functions on the first $k$ bits when $k = \omega{(\log n \log \log n)}$.  It is conjectured that there is no efficient algorithm for PAC-learning noisy parity that runs in time $o(2^{\sqrt{n}})$, which would imply hardness of noisy PAC learning parities on the first $k$ bits for $k = \omega{(\log^2 n)}$.  Calling this the `parity hardness assumption,'  Theorem~\ref{thm:parities} implies the following:

\begin{corollary}\label{cor:parities}
Under the parity hardness assumption, there is no efficient algorithm for PAC learning label proportions of parities on the first $k$ bits for $k = \omega{(\log^2 n)}$.
\end{corollary}

The above result relies on the hardness of PAC learning parities, rather than NP-hardness.  We can also establish problems that are NP-hard to learn.  To do this, we start by defining the consistency problem of a hypothesis class.
For a hypothesis class $C$, the \emph{consistency problem} for PAC learning from label proportions is the following:  Given a multi-set $X$ of points and an integer $k$, where each unique $x_i$ in $X$ appears some $a_i$ times in $X$, is there a hypothesis $c\in C$ such that the size of the multi-set of points that are labeled $1$ is exactly $k$?  That is, is there a hypothesis $c$ such that $\sum_{x_i:c(x_i)=1} a_i = k$?  If there is such a $c$, we will say that $c$ is consistent with $X$.

We reduce from the consistency problem to the learning problem, which is a slightly more involved reduction than in the classical PAC setting.

\begin{theorem}\label{thm:consistency_to_learning}
Suppose that the consistency problem for a hypothesis class $C$ is NP-hard.  There is no efficient algorithm for PAC learning $C$ from label proportions unless $\textup{NP} = \textup{RP}$.
\end{theorem}

\begin{proof}
It suffices to reduce from the consistency problem to the learning problem.  Indeed, using an oracle to an efficient PAC learner for $C$ from label proportions, we merely need to solve the consistency problem with high probability.  Given an instance of the consistency problem with input multi-set $X=\{x_i\}$ and integer $k$, define a distribution $D$ that outputs each unique $x_i$ with probability proportional to $a_i$.


Set \[\epsilon = \frac{1}{2|X|}.\] 
For $\delta > 0$, we will query the oracle with inputs $\delta$, $\epsilon$, and an i.i.d.\ sample from $D$ of size $m=f(1/\delta,1/\epsilon)$, where $f$ is the polynomial sample bound for the oracle.  Since the sample from $D$ may not be exactly $a_i$ copies of $x_i$, we do not know $\hat{p}$ to give to the oracle.  So instead, we will invoke the oracle $m+1$ times, setting the input proportion to be each of $0,1/m,\ldots,1$, and then check the resulting output hypothesis to see if it is consistent with $X$.  If so, accept, and if no such hypothesis is ever found, reject\footnote{The oracle's behavior is undefined if the value input as the proportion of positive labels is not the true value $\hat{p}$.  We may assume, however, that the oracle rejects whenever this is the case because the time the oracle takes is polynomially-bounded so we can just wait for that amount of time to see if the oracle returns a hypothesis.}.  

Certainly, if we accept, there is a consistent hypothesis by definition:  we accept if an oracle outputs a consistent hypothesis.  Conversely, if there is a consistent hypothesis $c$, then the oracle will accept:  Let $c$ be the consistent hypothesis.  Since it is consistent, by the definition of $D$, $p_c=k/|X|$.  Now consider the invocation of the oracle with the true proportion $\hat{p}$.  This invocation will output some hypothesis $h$ that will, except with probability at most $\delta$, satisfy 
\[|p_c-p_h| =  \left| \frac{k}{|X|} - \frac{\sum_{x_i\in S} a_i}{|X|} \right| \le \frac{1}{2|X|},\]
where $S$ is the set of points $h$ labels positively.  Since each $a_i$ is an integer, this implies that $\frac{k}{|X|}=\frac{\sum_{x_i\in S'} a_i}{|X|}$, i.e.~that $h$ is consistent with $X$ and therefore we will accept with probability at least $1-\delta$.

Setting $\delta$ to go to $0$ in the size of the input of the consistency problem completes the proof.
\end{proof}

Now, we can show that there are classes with linear VC dimension that are NP-hard to learn.  In particular, we start with monotone disjunctions.  Recall a monotone disjunction is a disjunction over $n$ variables without negations, i.e.~$\bigvee_{j\in J} x^j$ for a subset $J\subset[n]$, where we use $x^j$ to refer to the $j$th bit of a vector $x$.  We reduce the consistency problem for monotone disjunctions from a decision problem we will refer to as \textsc{Exact Partial Set Cover}, which asks, given a universe $U$, a collection of subsets $S\subset 2^U$ and an integer $k$, is there a subfamily $S'\subset S$ such that $|\cup_{s\in S'} s| = k$.
\begin{lemma}
\textsc{Exact Partial Set Cover} is NP-hard.
\end{lemma}
\begin{proof}
The reduction is from the well-known NP-hard problem \textsc{Exact Cover by 3-Sets}, \textsc{X3C}, which asks, given a universe $U$ of exactly $3t$ elements, and a collection $S$ of 3-element subsets of $U$, if there is a sub-collection $S'\subset S$ of size $t$ such that $\cup_{s\in S'} s = U$.  For each 3-element subset $s_i\in S$, define auxiliary elements $z_{i,1},\ldots,z_{i,\ell}$, with $\ell$ to be determined shortly.  Given an instance of \textsc{X3C}, we construct an instance of \textsc{Exact Partial Set Cover}.  The universe we define as $U\cup\{z_{i,j}\}$ and the collection of subsets we define as the collection $\{s_i\cup\{z_{i,1},\ldots,z_{i,\ell}\}\}$.  Finally, let $k=|U|+\ell t$.  This is a polynomial-time reduction as long as $\ell$ is only polynomially large.

If there is an exact cover by 3-sets, then the corresponding collection of subsets will have union exactly $|U|+\ell t$:  The non-auxiliary elements of this union must cover $U$, and since it is a size $t$ collection, there are $\ell t$ auxiliary elements in the union.  Conversely, if there is an exact partial set cover of size $b$ whose union is size $k=|U|+\ell t$, this collection must be exactly size $t$:  If $b>t$, then the size of the union, which is at least $\ell b$, must be strictly larger than $k$ as long as $\ell > |U|$.  And if $b<t$, then the size of the union is no more than $|U| +\ell b$, strictly smaller than $k$.  Thus the corresponding collection of subsets must be an exact cover, as it is size $t$ and its union is size $|U|$. 
\end{proof}

\begin{corollary}
Learning monotone disjunctions from label proportions is NP-hard.  
\end{corollary}
\begin{proof}
Using Theorem~\ref{thm:consistency_to_learning}, it suffies to reduce the consistency problem from \textsc{Exact Partial Set Cover}.  Given an instance $(U,S,k)$ of this problem, we construct an instance of the consistency problem $m=|U|$ examples $x_i$, each with $|S|$ bits.  Without loss of generality, we assume that $U=\{1,\ldots,m\}$.  We define the $j$th bit of $x_i$ to be $1$ if the $j$th set $s_j$ of the collection includes $i$, and otherwise we set the bit to be $0$.  We set the number of points to be labeled 1 as $k$.

The disjunction $d(x)  = x^{j_1} \vee \ldots \vee x^{j_\ell}$ will label exactly $k$ of these $m$ examples positively if and only if the union of the collection $S'=\{s_{j_1}, \ldots, s_{j_\ell}\}$ is size exactly $k$:  If $d(x_i) = 1$, then $x_{i}^j = 1$ for at least one $j\in\{j_1,\ldots,j_\ell\}$, i.e.\ $i$ is in $S'$.  And if $d(x_i) = 0$ , then $x_{i}^j = 0$ for all $j\in\{j_1,\ldots,j_\ell\}$, i.e.\ $i$ is not in any of the sets in $S'$.  Thus, the number of examples that $d$ labels positively is the size of the union of $S'$.
\end{proof}

Using a similar construction, it can be shown that learning monotone conjunctions is also NP-hard:
\begin{corollary}
Learning monotone conjunctions from label proportions is NP-hard.  
\end{corollary}

\section{Classes efficiently PAC learnable from label proportions}\label{sec:learnable}

In Section~\ref{sec:hardness}, we gave a number of examples of classes that are hard to LLP learn.  We now turn to examples of classes that can be efficiently LLP learned.  Certainly, as long as labelings in a given hypothesis class are efficiently enumerable, then finite classes $H$ are certainly PAC learnable from label proportions in time $|H|$.  Or instead, by enumerating only distinct hypotheses on the sample, assuming that this is efficient, learning can be achieved in $m^{d}$ time using Sauer's lemma.  This immediately implies that all such classes with constant $d$ are learnable from label proportions.  But there are also examples of classes with finite but larger VC dimension that can be efficiently LLP learned.

Consider the following hypothesis class which only allows hypotheses whose positive labels are close to each other:
\[H_k = \{h:\{1,\ldots,2^n\}\rightarrow\{0,1\} : \max_{h(i)=h(j)=1}|i-j| \le k\}.\]

There are still exponentially many functions and $VC(H_k)=k$.  For $k$ sufficiently small, this class is efficiently learnable:
\begin{observation}
PAC learning $H_k$ from label proportions has an $O(2^km)$ time algorithm.  
\end{observation}
Order the $m$ examples in $\{1,\ldots,2^n\}$, and for each length $k$ subset, of which there are $m-k+1$ of them, check all $2^k$ possible labelings.  Now when $k=O(\log n)$, this is a polynomial-time algorithm for learning $H_k$ from label proportions even though the VC dimension is not constant.

In the classical PAC setting, when it is hard to learn under an arbitrary distribution, it is often still valuable to show that learning can still be done in special cases, such as the uniform distribution.  We now give an example, namely half-spaces, where it is easy to learn under the uniform distribution.

The idea to find a half-space that classifies the given proportion $\hat{p}$ positively is to take a random half-space through the origin, and then move it in the direction of its normal vector, and stop when the half-space classifies the input $p$ proportion of the sample positively.  With high probability, this will be possible because no two points in the sample will be projected to the same point on the normal vector.

\begin{proposition}
The class of half-spaces in $n$ dimensions is learnable from label proportions under the uniform distribution over $\{0,1\}^n$.
\end{proposition}
\begin{proof}
Since the VC-dimension of half-spaces is linear in $n$ by Radon's theorem~\cite{mohri2012foundations}, using Theorem~\ref{thm:vc_gen_bound} it certainly suffices to be able to efficiently find a half-space $h$ such that $\hat{p_h} = p$ with high probability.  Consider a hyperplane $P$ of dimension $n-1$ through the origin and $v$ a normal vector defining $P$.  

Our goal will be to find such a vector $v$ such that no two points in $\{0,1\}^n$ project more than exponentially close to each other (in terms of $n$) on $v$.  This allows us to use only a polynomial number of bits to represent each projected point while still being able to find a hyperplane that separates every pair of consecutive projected points.  

It suffices to choose $v$ uniformly at random.  To show this, consider an arbitrary pair of points $x$ and $y$ in $\{0,1\}^n$ and consider the line $\ell$ that passes through these two points.  If $v$ and $\ell$ are perpendicular, then $x$ and $y$ will project onto the same point on $v$.  More generally, we can find the maximum angle between $v$ and $\ell$ so that the distance as a function of $n$, let's call this $d'(n)$, between the two projected points is not too small.  In particular, suppose $d'(n) = o(1/2^{n^c})$ for all constants $c>0$.  A distance of $o(1/2^{n^c})$ would require $\omega(n^c)$ bits to distinguish the two points.  For $x$ and $y$ distance $d$ apart, this maximum angle between $v$ and $\ell$ is $\sin^{-1}(d'(n)/d) = O\left(d'(n)/d\right)$ using the Taylor approximation for $\sin^{-1}(x)$.  

There are $O(2^{n^2})$ pairs of points since the points come from $\{0,1\}^n$.  Then the total sum of angles that would result in any pair of projected points being too close is $O\left(\frac{2^{n^2} d'(n)}{d}\right)$, which goes to $0$ exponentially quickly in $n$ because $d$ is at least a constant and $d'(n)=o(1/2^{n^2})$.  Thus, with high probability, no two points in $\{0,1\}^n$ project to the same point on $v$, or project more than exponentially close to each other on $v$.  

Given $m$ examples, setting $m$ to be polynomial in $n$ ensures with high probability that all examples are distinct, and therefore no two examples project more than exponentially close to each other on $v$.  Since $\hat{p}_c = i/m$ for some $i\in\{0,1,\ldots,m\}$, we need to find a plane parallel to $P$ such that the corresponding linear threshold function classifies $i$ of the sample points positively.  For each pair of consecutive projected points $cv$ and $c'v$ on $v$ for real number $c$ and $c'$, consider the half-space given by the plane defined by the points $p\in\mathbb{R}^n$ satisfying
$v\left(p - \left(\frac{c+c'}{2}\right)v\right) = 0,$ so that these two points are classified differently by the half-space.  Thus one of these half-spaces (or the half-spaces classifying all points positively or negatively) will have $\hat{p}_h = i/m$ since no two points in the sample project onto the same point on $v$.\end{proof}

While we have shown that it is strictly harder to PAC learn from label proportions than to PAC learn, introducing noise to the models changes the relationship between these two models.  For example, PAC learning parities with unknown $\eta$ white-label noise is hard under the uniform distribution, as discussed above, but PAC learning parities from label proportions with white-label noise is easy under the uniform distribution.  In our model, that means each label is flipped i.i.d.~with probability some unknown $\eta$, and the proportion of noisy positive labels $\hat{p}^\eta$ is given as input instead, but otherwise the learning requirement remains stays the same.

\begin{observation}
The class of parities is learnable from label proportions under the uniform distribution and unknown $\eta$ white-label noise.
\end{observation}
\begin{proof}
Let $p^\eta_c$ be the proportion of positive labels under $\eta$ noise and parity $c$.  Note $p^\eta_c$ is always 
\[(1-\eta)p_c+\eta(1-p_c) = p_c(1-2\eta)+\eta,\] but for any non-trivial parity $c$, $p_c = 1/2$, so $p^\eta_c = 1/2$.  Then Observation~\ref{obs:gap_sample_complexity} implies that we may distinguish efficiently the trivial parity from the non-trivial parities and in the case that $p^\eta_c = 1/2$ we may return any non-trivial parity.
\end{proof}

\section{Conclusion}

In this paper we formalized a model for learning a hypothesis class by only examples drawn from a distribution and the proportion of them receiving
each label, with the goal of finding a hypothesis that matches these statistics on the underlying distribution.

We give some initial results into a learning theory for this task, including that if a class with finite VC dimension is efficiently learnable from label proportions, it is automatically also efficiently properly PAC learnable.
On the other hand, we give an independence result that implies that 
learning with label proportions will not admit a nice VC-like characterization.
We also give examples where it is possible to efficiently PAC learn from label proportions, which may be surprising given that this is a low-information setting, including half-spaces under the uniform distribution.

These results are for the binary setting and only for the `one bag' version of the problem.  We leave  for future work the analysis of the case where there is more than one bag of examples and each bag's proportion of labels is given.  For that case, and in other similar settings where the 
learner is given more information, we expect there to be more positive algorithmic results.

\section*{Acknowledgements}
We thank Avrim Blum for pointing out that the proof and statement of Theorem 5 are incorrect in the version of this paper that was published in IJCAI 2017~\cite{fish2017llp};
in addition to including new results, we've corrected this error herein.
This work was supported in part by NSF awards CCF-1848966 and CCF-1934915.

\bibliographystyle{plain}
\bibliography{llp_paper}

\end{document}